\documentclass[twoside,11pt]{article}

\usepackage{blindtext}

\usepackage{url}

\usepackage{graphicx}
\usepackage{dirtytalk}
\usepackage{amsmath,amssymb,amsfonts}
\usepackage{algorithmic}
\usepackage{algorithm2e}

\DeclareMathOperator*{\argmax}{arg\,max}

\usepackage[abbrvbib]{dmlr2e}

\usepackage{lastpage}
\dmlrheading{23}{2024}{1-\pageref{LastPage}}{9/23; Revised 10/23}{11/23}{21-0000}{Julian Rodemann} 
\def\openreview{\footnotesize{\url{https://openreview.net/forum?id=WC5yWmgCGj&noteId=WC5yWmgCGj}}} 

\ShortHeadings{Bayesian Data Selection}{Bayesian Data Selection}
\firstpageno{1}

\begin{document}

\title{Towards Bayesian Data Selection}

\author{\name Julian Rodemann \email j.rodemann@lmu.de \\
       \addr Department of Statistics, LMU Munich\\
  }


\maketitle

\begin{abstract}

A wide range of machine learning algorithms iteratively add data to the training sample. Examples include semi-supervised learning, active learning, multi-armed bandits, and Bayesian optimization. We embed this kind of data addition into decision theory by framing data selection as a decision problem. This paves the way for finding Bayes-optimal selections of data. 
For the illustrative case of self-training in semi-supervised learning, we derive the respective Bayes criterion. We further show that deploying this criterion mitigates the issue of confirmation bias by empirically assessing our method for generalized linear models, semi-parametric generalized additive models, and Bayesian neural networks \mbox{on} simulated and real-world data. 

\end{abstract}


\section{Introduction}


The \textit{i.i.d.} assumption is arguably the most popular one in machine learning and statistics. It assumes that samples are independently drawn from the same distribution and is widely used in both statistical inference and machine learning. However, many real-world scenarios involve data with complex dependencies, distributional shifts, or non-stationary distributions. Moving \say{beyond \textit{i.i.d.}} \citep{beyondii} has thus become a broadly accepted research direction.      
Here, we consider the peculiar case where distributional shifts are induced by machine learning algorithms themselves. This happens if algorithms interact with the sample that they are iteratively trained on. In each iteration, such algorithms add new unseen data to the training data. That is, they not only learn parameters from data -- as is customary in machine learning -- but also vice versa. We refer to such methods as \textit{reciprocal learning}, see also \citep{reciprocal2024}. Examples comprise semi-supervised learning, active learning, multi-armed bandits, and Bayesian optimization. 

In this extended abstract, we showcase the potential of embedding them into decision theory by formalizing the sample enhancement procedure as a decision problem with a state space, an action space, and an integrable utility function. Leveraging classical results from Bayesian decision theory, we derive Bayes-optimal selections of data to be added to the training data in the case of self-training in semi-supervised learning. Since such a Bayesian treatment of data selection incorporates more than one model, it can be considered more robust than classical, non-Bayesian approaches. Moreover, it allows for a better understanding of the algorithm-induced distribution shift: By selecting data with regard to an integrable utility function, the latter provides us with the respective inclusion probabilities. These latter can be used for importance sampling which allows for unbiased inference in the presence of distribution shifts \citep{xie2023data,malte-debiased,rodemann2022not,james2023resampling,zadrozny2004learning}.

\section{Selecting Data is a Decision Problem}

There exist various criteria on how to select unseen data to be added to the training data. However, albeit being widely popular, most of these criteria lack rigorous derivation, let alone inferential guarantees. To address this research gap and allow for a Bayesian treatment, we incorporate reciprocal learning into decision theory. We frame the selection of unseen data as a decision problem, where the learner's parameter space represents the unknown set of states of nature, and the action space corresponds to the set of still unlabeled data points. This perspective allows us to adopt decision-theoretic approaches, particularly focusing on finding Bayes-optimal actions (i.e., the optimal selections of pseudo-labels) under common loss/utility functions.
Consider any iterative procedure, where new training instances from a set of unseen data $\mathcal{U}$ are added to the training data $\mathcal{D}$. In each iteration, the algorithm is faced with a decision problem: It has to select a subset of $\mathcal{U}$ to be added to $\mathcal{D}$.

\begin{definition}[Canonical Decision Problem]
\label{def:dec-probl}    
Define $(\mathbb{A}, \Theta, u(\cdot))$ as a decision-theoretic triple with an action space $\mathbb{A}$, an unknown set of states of nature $\Theta$ and a utility function $u : \mathbb{A} \times \Theta \to \mathbb{R}$. 

\end{definition}

In this framework, an action corresponds to the selection of an instance from the unseen data. This is in stark contrast to statistical decision theory, where statistical procedures instead of data are to be selected, see \citep{berger1985statistical} for instance. The decision for an action is guided by a utility function that depends on the concrete reciprocal learning algorithm. For instance, in Bayesian optimization, the selection is typically guided by the expected improvement \citep{mockus1978application} of the target function in the optimization problem. In the following, we illustrate this general framework with the specific case of self-training in semi-supervised learning.

\section{Example: Bayesian PLeaSe! Bayesian Pseudo-Label Selection (PLS) }\label{sec:BPLS}

In many practical classification scenarios, obtaining labeled data can be quite challenging. This has led to the emergence of semi-supervised learning (SSL), a paradigm that leverages information from unlabeled data to enhance inference derived from labeled data in a supervised learning framework. Within SSL, a commonly used and intuitive approach is known as self-training or pseudo-labeling \citep{shi2018transductive,lee2013pseudo,rizve2020defense}. The idea behind pseudo-labeling is to train an initial model using labeled data and then iteratively assign pseudo-labels to unlabeled data based on the model's predictions. The above-introduced training data 
$
    \mathcal{D}=\left\{\left(x_{i}, y_{i}\right)\right\}_{i=1}^{n} \in \left(\mathcal{X} \times \mathcal{Y}\right)^{n}
$ is labeled and the unseen data is unlabeled $  \mathcal{U}=\left\{\left(x_{i}, \mathcal{Y}\right)\right\}_{i=n+1}^{m} \in \left(\mathcal{X} \times 2^\mathcal{Y}\right)^{m-n}
$ and from the same data generation process. Let $\mathcal{X}$ be the feature space and $\mathcal{Y}$ the categorical target space, whereby unlabeled data are notationally equated with observing the full ${\cal Y}$. The aim of SSL is to learn a predictive classification function $f$ such that $f( x) = \hat y \in \mathcal{Y}$ utilizing both $\mathcal{D}$ and $\mathcal{U}$.  
As is customary in self-training, we start by fitting a model with parameter vector $\theta \in \Theta$, where $\Theta$ is closed and bounded with $\dim(\Theta) = q$, on labeled data $
    \mathcal{D}=\left\{\left(x_{i}, y_{i}\right)\right\}_{i=1}^{n} $.
Our goal is -- as usual -- to learn the conditional distribution of $p( y \mid x)$ through $\theta$ from observing features ${x} = (x_1, \dots, x_n) \in \mathcal{X}^n$, and responses ${y} = (y_1, \dots, y_n) \in \mathcal{Y}^n$ in $\mathcal{D}$. 
We define a prior probability distribution over $\theta$ as $\pi(\theta)$. 

\RestyleAlgo{ruled}
\SetKwComment{Comment}{/* }{ */}

\begin{algorithm}[H]
\caption{Self-Training (\textit{alias} Pseudo-Labeling, Self-Labeling)}
\label{alg:main}

\KwData{$\mathcal{D}, \mathcal{U}$}
\KwResult{fitted model $\hat y^*(x)$}

\While{stopping criterion not met}{
\textbf{fit} model on labeled data $\mathcal{D}$ to obtain prediction function $\hat y(x)$ \\
\For{$i \in \{1, \dots, \lvert \mathcal{U} \rvert \}$}{
\textbf{predict} $\mathcal{Y} \ni \hat y_i = \hat y(x_i)$ with $x_i$ from $\left(x_{i}, \mathcal{Y}\right)_i$ in $\mathcal{U}$ \\
\textbf{compute} some selection criterion $c(x_{i}, \hat y_i)$
\\
}
\textbf{obtain} $i^* = \argmax_i \, c(x_{i}, \hat y_i) \, $ \\ 
\textbf{add} $(x_{i^*}, \hat y_{i^*})$ to labeled data: $\mathcal{D} \leftarrow \mathcal{D} \cup (x_i, \hat y_i) $ \\
\textbf{update} $\mathcal{U} \leftarrow \mathcal{U} \setminus \left(x_{i}, \mathcal{Y}\right)_i $
}
\end{algorithm}

The process of selecting pseudo-labeled instances, known as pseudo-label selection (PLS), is guided by some \textit{ad hoc} criteria such as the predictive variance or the predicted class probabilities \citep{arazo2020pseudo,Sohn2020,shi2018transductive,lee2013pseudo,mcclosky2006effective}. The choice of those typically lacks theoretical justification. We address this by embedding PLS into decision theory. Building on definition \ref{def:dec-probl}, the action space corresponds to unlabeled data, i.e.,  $\mathbb{A}_{\mathcal{U}} = \{(z, \mathcal{Y}) \mid \exists \, i \in \{n+1, \dots, m\} : (z, \mathcal{Y}) = (x_i, \mathcal{Y})_i \in \mathcal{U}    \}$. We further define the utility of a selected data point $(z, \mathcal{Y}) = (x_i, \mathcal{Y})_i$ as the plausibility of being generated jointly with $\mathcal{D}$ by a model $M$ with states (parameters) $\theta \in \Theta$ if we include it with its predicted pseudo-label $\hat y(z) = \hat y(x_i) = \hat y_i \in \mathcal{Y}$ in $\mathcal{D} \cup (x_i, \hat{y}_i)$.

\begin{definition}[Pseudo-Label Likelihood as Utility]
\label{def:pseud-lik}
Given $\mathcal{D}$ and the prediction functional $\hat y: \mathcal{X} \to \mathcal{Y}$, we define the following utility function\begin{align*}
  u \colon \mathbb{A}_{\mathcal{U}} \times \Theta &\to \mathbb{R}\\
  ((z, \mathcal{Y}), \theta) &\mapsto u((z, \mathcal{Y}), \theta) = p(\mathcal{D} \cup (z, \hat{y}(z))\mid \theta, M),
  \end{align*}
  which is said to be the pseudo-label likelihood.
\end{definition}

It can now be shown that by using the joint likelihood as utility, the Bayes-optimal criterion for PLS is the posterior predictive of pseudo-samples and labeled data.

\begin{theorem} [Pseudo Posterior Predictive \citep{rodemann2023-bpls}]
\label{th:ppp}
In the decision problem $(\mathbb{A}_{\mathcal{U}}, \Theta, u(\cdot))$ and the pseudo-label likelihood as utility function (definition~\ref{def:pseud-lik}) and the prior updated by the posterior $\pi(\theta) = p(\theta \mid \mathcal{D})$ on $\Theta$, the Bayes criterion 
$\Phi(\cdot, \pi) \colon \mathcal{U} \to \mathbb{R}; \, a \mapsto \Phi(a, \pi) = \mathbb{E}_\pi(u(a,\theta)) $
corresponds to $p(\mathcal{D} \cup (x_i, \hat{y}_i)\mid \mathcal{D})$, which is the posterior predictive of pseudo-samples and labeled data and shall be called \textit{pseudo posterior predictive}.
\end{theorem}

\begin{proof}
    See Appendix~\ref{app:theorem}.
\end{proof}

To make this criterion computationally feasible, approximations based on Laplace's method and the Gaussian integral are required, avoiding the need for expensive sampling-based evaluations, see \citep[chapter 3.2]{rodemann2023-bpls}. The approximate Bayes-optimal criterion becomes simple to compute: $\ell(\hat{\theta}) - \frac{1}{2}\log\lvert \mathcal{I}(\hat{\theta})\rvert$, where $\ell(\hat{\theta})$ denotes the likelihood and $\mathcal{I}(\hat{\theta})$ the Fisher-information matrix at the fitted parameter vector $\hat{\theta}$. This approximation does not rely on the assumption of \textit{i.i.d.} data, making it widely applicable to various learning tasks. 
What is more, the approximation allows for the computation of the Bayes criterion even in more involved extensions such as considering a multi-objective utility. 
In \citep{inalllikelihoods,dietrich2024}, several ways of dealing with such a multi-objective utility are discussed. For instance, an embedding into a \textit{preference system} \citep{jsa2018,jansen2023statistical,uai2023_all} is easy due to the decision-theoretic formalization presented above. 
\citep{dietrich2024} further consider the generic approach of the generalized Bayesian $\alpha$-cut updating rule \citep{cattaneo2014continuous} for credal sets of priors. Such sets of priors can not only reflect uncertainty regarding prior information but might as well represent priors near ignorance, see~\citep{mangili2015new,mangili15prior,rodemann-BO-isipta,prior-robust-BO,imprecise-BO} for instance. Results suggest that considering this additional source of uncertainty pays out, particularly when the share of labeled data is low \citep{dietrich2024}.   
In \citep{rodemann2023-bpls}, Bayesian PLS (BPLS) is empirically tested on both simulated and real-world datasets, using different models such as generalized linear, semi-parametric generalized additive models, and Bayesian neural networks. Empirical findings suggest that BPLS can effectively mitigate the confirmation bias arising from overfitting initial models during pseudo-labeling.\footnote{All results and scripts to reproduce Bayesian PLS are here: \url{https://anonymous.4open.science/r/Bayesian-pls}. For robust extensions of Bayesian PLS including $\alpha$-cut updating, see \url{https://anonymous.4open.science/r/reliable-pls-01CF}.} 
Moreover, BPLS offers flexibility in incorporating prior knowledge not only for prediction but also for selecting pseudo-labeled data. Remarkably, BPLS does not involve any hyperparameters that require tuning. Additionally, the decision-theoretic treatment of PLS naturally extends to the framework of optimistic and pessimistic superset learning \citep{hullermeier2014learning,hullermeier2019learning,rodemann2022supersetlearning} as max-max and min-max actions, respectively.

\section{Outlook: Decision Criteria Allow Importance Sampling}

As follows from Theorem \ref{th:ppp}, the Bayes criterion for data selection is a posterior predictive \textit{probability} distribution. We thus have explicit inclusion probabilities for the sample enhancement mechanism to address the issue that self-selected data is not \textit{i.i.d.}. This allows for debiasing estimators through inverse probability weighting, see \citep[Section 7]{malte-debiased} for the case of random forests. Venues for future work include deriving inclusion probabilities for reciprocal learning beyond self-training and de-biasing estimators by procedures beyond inverse probability weighting. The latter appears particularly crucial as inverse probability weighting is prone to inflating estimators' variances. Furthermore, inferential guarantees for such estimators have yet to be established.




\newpage
\impact{The herein discussed methods contribute at making data selection more robust towards a broad spectrum of uncertainties. This can help rendering machine learning more reliable and trustworthy, but may also be misused for adverserial purposes.}

\section*{Reproducibility Statement}
As mentioned in section~\ref{sec:BPLS}, all results and scripts to reproduce them can be found here: \url{https://anonymous.4open.science/r/Bayesian-pls}. For robust extensions of Bayesian PLS including $\alpha$-cut updating, see \url{https://anonymous.4open.science/r/reliable-pls-01C}.



\newpage
\appendix
\section{}
\label{app:theorem}



In this appendix we prove the following theorem from
Section~3:\\

\noindent
\textbf{Theorem} In the decision problem $(\mathbb{A}_{\mathcal{U}}, \Theta, u(\cdot))$ and the pseudo-label likelihood as utility function (definition \ref{def:pseud-lik} and the prior updated by the posterior $\pi(\theta) = p(\theta \mid \mathcal{D})$ on $\Theta$, the Bayes criterion 
$\Phi(\cdot, \pi) \colon \mathcal{U} \to \mathbb{R}; \, a \mapsto \Phi(a, \pi) = \mathbb{E}_\pi(u(a,\theta)) $
corresponds to the \textit{pseudo posterior predictive} $p(\mathcal{D} \cup (x_i, \hat{y}_i)\mid \mathcal{D})$.\\

\begin{proof}
    The definition of the expected value for integrable $u(\cdot, \cdot)$ directly delivers \[ \Phi(a, \pi) = \mathbb{E}_\pi(u(a,\theta)) = \int u(a, \theta) d \pi(\theta) = \int p(\mathcal{D} \cup (x_i, \hat{y}_i)\mid\theta) d \pi(\theta) = p(\mathcal{D} \cup (x_i, \hat{y}_i)).\]
    
 With the updated prior $\pi(\theta) = p(\theta \mid \mathcal{D})$ it follows \[ \int u(a, \theta) d \pi(\theta)= \int p(\mathcal{D} \cup (x_i, \hat{y}_i)\mid\theta) d p(\theta \mid \mathcal{D}) = p(\mathcal{D} \cup (x_i, \hat{y}_i)\mid \mathcal{D}).\]
 
\end{proof}

\newpage
\vskip 0.2in
\bibliography{sample}

\end{document}